\documentclass{eptcs}
\usepackage{latexsym}
\usepackage{amssymb}
\usepackage{amsmath}
\usepackage{color}
\usepackage[all, knot]{xy}
\usepackage{stmaryrd}
\usepackage[lined,boxed,commentsnumbered]{algorithm2e}

\definecolor{green}{rgb}{0.0, 0.5, 0.0}
\definecolor{lightbrown}{rgb}{0.71, 0.4, 0.11}

\newtheorem{definition}{Definition}[section]
\newtheorem{proposition}{Proposition}[section]

\newtheorem{corollary}{Corollary}[section]
\newtheorem{claim}{Claim}[section]
\newtheorem{theorem}{Theorem}[section]
\newtheorem{lemma}{Lemma}[section]
\newtheorem{example}{Example}

\newtheorem{remark}{Remark}

\newcommand{\LRA}{\Leftrightarrow}

\newcommand{\PSPACE}{\textsc{Pspace}}

\newcommand{\EXPTIME}{\textsc{EXPtime}}

\newcommand{\EXPSPACE}{\textsc{EXPspace}}

\newenvironment{proof}{\noindent{\bf Proof:}}{\hfill\rule{2mm}{2mm}\\}

\newcommand{\tP}{{\texttt{P}}}
\newcommand{\tA}{{\texttt{A}}}
\newcommand{\EAL}{{\texttt{EAL}}}
\newcommand{\tEPDL}{{\texttt{EPDL}}}
\newcommand{\tUM}{{\texttt{UM}}}
\newcommand{\tB}{{\texttt{B}}}
\newcommand{\B}{{\texttt{B}}}

\newcommand{\calS}{{\cal S}}
\newcommand{\calR}{{\cal R}}
\newcommand{\calV}{{\cal V}}
\newcommand{\calU}{{\cal U}}
\newcommand{\calM}{{\cal M}}
\newcommand{\calN}{{\cal N}}

\newcommand{\N}{{\cal N}}
\renewcommand{\S}{{\cal S}}
\newcommand{\R}{{\cal R}}
\newcommand{\V}{{\cal V}}
\newcommand{\U}{{\cal U}}

\renewcommand{\L}{{\cal L}}
\newcommand{\EPDL}{{\texttt{EPDL}}}
\newcommand{\EPDLsf}{{\texttt{EPDL$^-$}}}

\newcommand{\lrp}[1]{\llparenthesis #1 \rrparenthesis}

\newcommand{\etsM}{\mathfrak{M}}
\newcommand{\llrr}[1]{\llbracket #1\rrbracket}

\newcommand{\M}{\mathcal{M}}

\newcommand{\lr}[1]{\langle #1 \rangle}
\newcommand{\lra}{\leftrightarrow}

\newcommand{\rel}[1]{\stackrel{#1}{\rightarrow}}

\newcommand{\tr}[1]{\text{#1}}

\newcommand{\SELA}{\mathbb{SELA}}

\newcommand{\DIST}[1]{\ensuremath{\mathtt{DIST}(#1)}}
\newcommand{\DISTK}{\ensuremath{\mathtt{DISTK}}}
\newcommand{\TAUT}{\ensuremath{\mathtt{TAUT}}}
\newcommand{\GEN}[1]{\ensuremath{\mathtt{NEC}(#1)}}
\newcommand{\GENK}{\ensuremath{\mathtt{NECK}}}

\newcommand{\MP}{\ensuremath{\mathtt{MP}}}
\newcommand{\ZIG}[1]{\ensuremath{\mathtt{PR}(#1)}}
\newcommand{\ZAG}[1]{\ensuremath{\mathtt{NM}(#1)}}
\newcommand{\AxTr}{\ensuremath{\mathtt{T}}}
\newcommand{\AxTrans}{\ensuremath{\mathtt{4}}}
\newcommand{\AxEuc}{\ensuremath{\mathtt{5}}}
\newcommand{\OBS}[1]{\ensuremath{\mathtt{OBS}(#1)}}
\newcommand{\SUB}{\ensuremath{\mathtt{SUB}}}
\newcommand{\xrel}{\rel}

\newcommand{\qt}{QT}

\def\titlerunning{A Dynamic Epistemic Framework for Conformant Planning}
\def\authorrunning{Yu, Li \& Wang}

\title{A Dynamic Epistemic Framework for Conformant Planning}

\author{Quan Yu \institute{Department of Computer Science, Sun Yat-sen University, China}\institute{Qiannan Normal College for Nationalities, China}
\and Yanjun Li \institute{Department of Philosophy, Peking University, China} \institute{ Faculty of Philosophy, University of Groningen, The Netherlands}\and  Yanjing Wang\thanks{Corresponding author} \institute{Department of Philosophy, Peking University, China}}

\begin{document}

\maketitle

\begin{abstract}In this paper, we introduce a lightweight dynamic epistemic logical framework for automated planning under initial uncertainty. We reduce plan verification and conformant planning to model checking problems of our logic. We show that the model checking problem of the iteration-free fragment is PSPACE-complete. By using two non-standard (but equivalent) semantics, we give novel model checking algorithms to the full language and the iteration-free language.
\end{abstract}

\section{Introduction}
\textit{Conformant planning} is the problem of finding a linear plan (a sequence of action) to achieve a goal in presence of  uncertainty about the initial state (cf. \cite{SW98}). For example, suppose that you are a rookie spy trapped in a foreign hotel with the following map at hand:\footnote{It is a variant of the running example in \cite{wang2012}.}

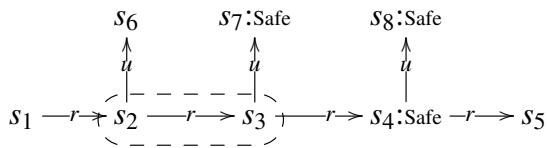
\begin{large}
$$\xymatrix{
&s_6&{{s_7\textrm{:\scriptsize{Safe}}}}&{{s_8\textrm{:\scriptsize{Safe}}}} &\\
s_1\ar[r]|r& s_2\ar[r]|r\ar[u]|u& s_3\ar[r]|r\ar[u]|u&{s_4\textrm{:\scriptsize{Safe}}}\ar[r]|r\ar[u]|u&s_5
\save "2,2"."2,3"!C="g1"*+[F--:<+20pt>]\frm{}
\restore
}$$
\end{large}

Now somebody spots you and sets up the alarm. In this case you need to move fast to one of the safe hiding places marked in the map (i.e., $s_7, s_8$ and $s_4$). However, since you were in panic, you lost your way and you are not sure whether you are at $s_2$ or $s_3$ (denoted by the circle in the above graph). Now what should you do in order to reach  a safe place quickly? Clearly, merely moving $r$ or moving $u$ may not guarantee your safety given the uncertainty. A simple plan is to move $r$ first and then $u$, since this plan will take you to a safe place, no matter where you actually are initially. This plan is \textit{conformant} since it does not require any feedback during the execution and it should work in presence of uncertainty about the initial state. More generally, a conformant plan should also work given actions with non-deterministic effects. Such a conformant plan is crucial when there are no feedbacks\slash observations available during the execution of the plan.\footnote{In many other cases, feedbacks may be just too `expensive' to obtain during a plan aiming for quick actions \cite{bonet2010conformant}.} Note that since no information is provided during the execution, the conformant plan is simply a finite sequence of actions without any conditional moves.

As discussed in \cite{BonetG00,PalaciosG06}, conformant planning can be reduced to classical planning, the planning problem without any initial uncertainty, over the space of \textit{belief states}. Intuitively, a belief state is a subset of the state space, which records the uncertainty during the execution of a plan, e.g., $\{s_2, s_3\}$ is an initial belief state in the above example. In order to make sure a goal is achieved eventually, it is crucial to track the transitions of belief states during the execution of the plan, and this may traverse exponentially many belief states in the size of the original state space. As one may expect, conformant planning is computationally harder than classical planning. The complexity of checking the existence of a conformant plan is \EXPSPACE-complete in the size of the variables generating the state space \cite{HaslumJ99}. In the literature, people proposed compact and implicit representations of the belief spaces, such as OBDD \cite{CimattiR00, CimattiRB04,Cimatti11} and CNF \cite{ToSP10}, and different heuristics are used to guide the search for a plan, 
e.g., \cite{BrafmanH04,BryceKS06}.
\medskip

Besides the traditional AI approaches, we can also take an epistemic-logical perspective on planning in presence of initial uncertainties, based on dynamic epistemic logic (DEL) (cf. e.g., \cite{DHK2007}). The central philosophy of DEL takes the meaning of an action as the change it brings to the knowledge of the agents. Intuitively, this is what we need to track  the belief states during the execution of a plan\footnote{Here the belief states are actually about knowledge in epistemic logic.}. Indeed, in recent years, there has been a growing interest in using DEL to handle multi-agent planning with knowledge goals (cf.\ e.g., \cite{BA2011,LEA11,andersen2012conditional,aucher2012sequents,YWL13,
pardo2013strong}), while the traditional AI planning focuses on the single-agent case. In particular, the event models of DEL (cf.\ \cite{BaMo04}) are used to handle non-public actions that may cause different knowledge updates to different agents. In these DEL-based planning frameworks, states are epistemic models, actions are event models and the state transitions are implicitly encoded by the update product which computes a new epistemic model based on an epistemic model and an event model.

One advantage of this approach is its expressiveness in handling scenarios which require reasoning about agents' higher-order knowledge about each other in presence of partially observable actions. However, this expressiveness comes at a price, as shown in \cite{BA2011,aucher2013undecidability}, that multi-agent epistemic planning is undecidable in general. Many interesting decidable fragments are found in the literature \cite{BA2011,LEA11,YWL13,andersen2013don}, which suggests that the single-agent cases and restrictions on the form of event models are the key to decidability. However, if we focus on the single-agent planning, a natural question arises: how do we compare such DEL approaches with the traditional AI planning? It seems that the DEL-based approaches are more suitable for planning with actions that change (higher-order) knowledge rather than planning with fact-changing actions, although the latter type of actions can also be handled in DEL. Moreover, the standard models of DEL are purely epistemic thus do not encode the temporal information of available actions directly. This may limit the applicability of such approaches to planning problems based on transition systems.

\medskip

In this paper, we tackle the standard single-agent conformant planning problem over transition systems, by using the core idea of DEL, but not its  standard formalism. Our formal framework is based on the logic proposed by Wang and Li in \cite{wang2012}, where the model is simply a transition system with initial uncertainty as in the motivating example, and an action is interpreted in the semantics as an update on the uncertainty of the agent. Our contributions are summarized as follows:\\


\begin{itemize}
\item  A lightweight dynamic epistemic framework with a simple language and a complete axiomatization. 
\item   Non-trivial reduction of conformant planning to a model checking problem using our language with programs.
\item  Two novel model checking algorithms based on two alternative semantics for the proposed logic,  which make the context-dependency in the original semantics explicit. 
\item The complexity of model checking the iteration-free fragment of our language is \PSPACE-complete. The model checking problem of the full language is in \EXPTIME. The model checking problem of the conformant planning is in \PSPACE. 
\end{itemize}

The last result may sound contradictory to the aforementioned result that the complexity of conformant planning is \EXPSPACE-complete. Actually, the apparent contradiction is due to the fact that the \EXPSPACE\ complexity result is based on the number of \textit{state variables} which require an exponential blow up to generate an explicit transition system that we use here. We will come back to this issue at the end of Section \ref{sec.cp}.

\medskip

Our approach has the following advantages compared to the existing planning approaches: \\

\begin{itemize}
\item The planning goals can be specified as arbitrary formulas in an epistemic language.  Extra plan constraints (e.g., what actions to use) can be expressed explicitly by programs in the language. Therefore it may cover a richer class of (conformant) planning problems compared to the traditional AI approach where a goal is Boolean.\footnote{The goal in the standard conformant planning is simply a set of different valuations of basic propositional variables. Our approach can even handle epistemic goals in negative forms, e.g., we want to make sure the agent knows something but does not know too much in the end. }
\item The plans can be specified as regular expressions with tests in terms of arbitrary EPDL formulas, which generalizes the knowledge-based programs in \cite{FHMV97:KbasedP,LangZ12}.
\item By reducing conformant planning to a model checking problem in an explicit logical language, we also see the subtleties hidden in the planning problem. In principle, there are various model checking techniques to be applied to  conformant planning based on this reduction.
\item Our logical language and models are very simple compared to the standard action-model based DEL approach, yet we can encode the externally given  executability of the actions in the model, inspired by epistemic temporal logic (ETL) \cite{RAK,ParikhIS}. 
\item Our approach is flexible enough to provide, in the future, a unified platform to compare different planning problems under uncertainty. By studying different fragments of the logical language and model classes, we may categorize planning problems according to their complexity.
\end{itemize}

The rest of the paper is organized as follows: We introduce our basic logical framework and its axiomatization in Section \ref{sec.basic}, and extend it in Section \ref{sec.pdl} with programs to handle the conformant planning. The complexity analysis of the model checking problems is in Section \ref{sec.mc} and we conclude in Section \ref{sec.conc} with future directions.

\section{Basic framework}\label{sec.basic}

\subsection{Epistemic action language}

To talk about the knowledge of the agent during an execution of a plan, we use the following language proposed in \cite{wang2012}.


\begin{definition}[Epistemic Action Language (\EAL)] Given a countable set \tA\ of action symbols  and a countable set \tP\ of atomic proposition letters , the language $\EAL^{\tA}_{\tP}$ is defined as follows:\footnote{We do need unboundedly many action symbols to encode the desired problem in the later discussion of model checking complexity. }
\[\phi ::= \top \mid p \mid \neg\phi \mid (\phi\wedge\phi) \mid [a ]\phi \mid K\phi,\]
where $p\in \tP$, $a \in \tA$. The following standard abbreviations are used: $\bot:=\neg\top$,  $\phi\vee\psi:=\neg(\neg \phi\wedge\neg \psi),\phi\rightarrow\psi:=\neg\phi\vee\psi,
\lr{a}\phi:=\neg [a]\neg\phi,
\hat{K}\phi:=\neg K\neg\phi$.
\end{definition}

$K\phi$ says that the agent knows that $\phi$, and $[a]\phi$ expresses that if the agent can move forward by action $a$, then after doing $a$, $\phi$ holds. Throughout the paper, we fix some $\tP$ and $\tA$, and refer to $\EAL^{\tA}_{\tP}$ by $\EAL$.\\

The size of $\EAL$-formulas (notation $|\varphi|$) is defined inductively:
$|\top|=|p|=1$; $|\neg \phi|=1+|\phi|$; $|\phi \wedge \psi|=1+ |\phi|+|\psi|$; $|K\phi|=|[a]\phi|=1+|\phi|$.
The set of subformulas of $\phi\in\EAL$, denoted as $sub(\phi)$, is defined as usual.

\begin{definition}[Uncertainty map] Given $\tP$ and $\tA$, a (multimodal) Kripke model $\calN$ is a tuple $\langle \calS, \{\calR_{a}\mid a\in \tA\}, \calV \rangle$, where $\calS$ is a non-empty set of states, $\calR_{a}\subseteq \calS \times \calS$ is a binary relation labelled by $a$,  $\calV: \calS \rightarrow 2^{\tP}$
is a valuation function. An uncertainty map $\M$ is a Kripke model $\langle \calS,\{\calR_{a}\mid a\in \tA\}, \calV \rangle$ with a non-empty set $\calU\subseteq \calS$. Given an uncertainty map $\M$, we refer to its components by $\S_\M$, ${\R_a}_{\M}$, $\V_\M$, and $\U_\M$. A pointed uncertainty map $\M,s$ is an uncertainty map $\M$ with a designated state $s\in \U_\M$. We write $s \rel{a} t$ for $(s,t)\in \calR_{a}$.
 \end{definition}

Intuitively, a Kripke model encodes a map (transition system) and the uncertainty set $\U$ encodes the uncertainty that the agent has about where he is in the map. The graph mentioned at the beginning of the introduction is a typical example of an uncertainty map. Note that there may be non-deterministic transitions in the model, i.e., there may be $t_1\not= t_2$ such that $s\rel{a}t_1$ and $s\rel{a}t_2$ for some $s, t_1, t_2$.

\begin{remark}
It is crucial to notice that the designated state in a pointed uncertainty map must be one of the states in the uncertainty set.
\end{remark}

\begin{definition}[Semantics]
Given any uncertainty map $\calM=\langle \calS, \{\calR_{a}\mid a\in \tA\}, \calV, \calU \rangle$ and any state $s\in \calU$, the semantics is defined as follows:
$$\begin{array}{|lcl|}
\hline
\M,s\vDash \top &  &\tr{always}\\
\M,s\vDash p &\iff& s\in \V(p)\\
\M,s\vDash \neg\phi &\iff& \M,s\nvDash \phi\\
\M,s\vDash \phi\wedge\psi &\iff& \M,s\vDash\phi \text{ and } \M,s\vDash \psi\\
\M,s\vDash [a]\phi &\iff& \forall t\in S: s\rel{a}t \text{ implies } \M|^a,t\vDash \phi\\
\M,s\vDash K\phi &\iff& \forall u\in\U:\M,u\vDash \phi\\
\hline
\end{array}
$$

\noindent where $\M|^a=\lr{\S,\{\R_{a}\mid a\in \tA\},
\V, \U|^a}$ and $\U |^a=\{r'\mid \exists r\in \U \text{ such that }r\rel{a}r' \}$. We say $\phi$ is valid (notation: $\vDash\phi$) if it is true on all the pointed uncertainty maps. For a action sequence $\sigma=a_1\dots a_n$, we write $\U|^{\sigma}$ for $(\dots((\U|^{a_1})|^{a_2})\dots)|^{a_n}$. and write $\M|^{\sigma}$ for $(\dots((\M|^{a_1})|^{a_2})\dots)|^{a_n}$.
\end{definition}

Intuitively, the agent `carries' the uncertainty set with him when moving forward and obtains a new uncertainty set $\U|^a$. Note that here we differ from \cite{wang2012} where the updated uncertainty set is further refined according to what the agent can observe at the new state. For conformant planning, we do not consider the observational power of the agent during the execution of a plan.

Let us call the model mentioned in the introduction $\M$, it is not hard to see that $\M|^r$ and $(\M|^r)|^u$ are as follows:
$$
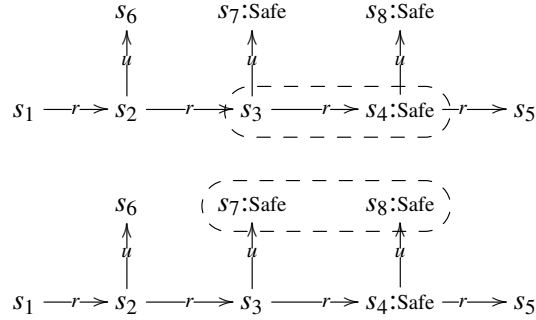
\xymatrix{
&s_6&{{s_7\textrm{:\scriptsize{Safe}}}}&{{s_8\textrm{:\scriptsize{Safe}}}} &\\
s_1\ar[r]|r& s_2\ar[r]|r\ar[u]|u& s_3\ar[r]|r\ar[u]|u&{s_4\textrm{:\scriptsize{Safe}}}\ar[r]|r\ar[u]|u&s_5
\save "2,3"."2,4"!C="g1"*+[F--:<+20pt>]\frm{}
\restore
}$$

$$\xymatrix{
&s_6&{{s_7\textrm{:\scriptsize{Safe}}}}&{{s_8\textrm{:\scriptsize{Safe}}}} &\\
s_1\ar[r]|r& s_2\ar[r]|r\ar[u]|u& s_3\ar[r]|r\ar[u]|u&{s_4\textrm{:\scriptsize{Safe}}}\ar[r]|r\ar[u]|u&s_5
\save "1,3"."1,4"!C="g1"*+[F--:<+20pt>]\frm{}
\restore
}$$

Thus we have:
\begin{itemize}
\item $\M,s_3\vDash [r] (\textit{Safe}\land \neg K \textit{Safe})$
\item $\M,s_3\vDash K [r][u] (\textit{Safe}\land K \textit{Safe})$
\end{itemize}

The usual global model checking algorithm for modal logics labels the states with the subformulas that are true on the states. However, this cannot work here since the truth value of epistemic formulas on the states outside $\U$ is simply  undefined. Moreover, the exact truth value of an epistemic formula on a state depends on `how you get there', as the following example shows (the underlined states mark the actual states):
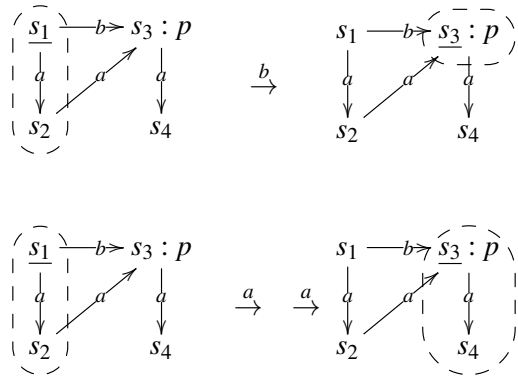
\begin{center}
\begin{minipage}{0.20\textwidth}
$$\xymatrix{\underline{s_1}\ar[r]|b\ar[d]|a&s_3:p\ar[d]|a\\
s_2\ar[ur]|a&s_4
\save "1,1"."2,1"!C="g1"*+[F--:<+20pt>]\frm{}
\restore}$$
\end{minipage}
\begin{minipage}{0.05\textwidth}
$$\rel{b}$$
\end{minipage}
\begin{minipage}{0.2\textwidth}
$$\xymatrix{s_1\ar[r]|b\ar[d]|a&\underline{s_3}:p\ar[d]|a\\
s_2\ar[ur]|a&s_4
\save "1,2"."1,2"!C="g1"*+[F--:<+20pt>]\frm{}
\restore}$$
\end{minipage}
\end{center}

\begin{center}
\begin{minipage}{0.20\textwidth}
$$\xymatrix{\underline{s_1}\ar[r]|b\ar[d]|a&s_3:p\ar[d]|a\\
s_2\ar[ur]|a&s_4
\save "1,1"."2,1"!C="g1"*+[F--:<+20pt>]\frm{}
\restore}$$
\end{minipage}
\begin{minipage}{0.05\textwidth}
$$\rel{a}\quad \rel{a}$$
\end{minipage}
\begin{minipage}{0.2\textwidth}
$$\xymatrix{s_1\ar[r]|b\ar[d]|a&\underline{s_3}:p\ar[d]|a\\
s_2\ar[ur]|a&s_4
\save "1,2"."2,2"!C="g1"*+[F--:<+20pt>]\frm{}
\restore}$$
\end{minipage}
\end{center}
\medskip

Let the left-hand-side model be $\M$ then it is clear that $\M|^b,s_3\vDash Kp$ while $\M|^{aa},s_3\nvDash Kp$ thus $\M,s_1\vDash \lr{b}Kp\land\lr{a}\lr{a}\neg Kp.$ This shows that the truth value of an epistemic subformula w.r.t. a state in the model is somehow `context-dependent', which requires new techniques in model checking. We will make this explicit in Section~\ref{semanticsth} when we discuss the model checking algorithm.
\medskip

\subsection{Axiomatization}

Following the axioms proposed in \cite{wang2012}, we give the following axiomatization for \EAL\ w.r.t.\ our semantics:

{\small \begin{center}
\begin{tabular}{lclc}
\multicolumn{4}{c}{System $\SELA$}\\
\multicolumn{2}{l}{\textbf{Axioms}}&\textbf{Rules}&\\
\TAUT & \tr{all axioms of propositional logic}&\MP & $\dfrac{\phi,\phi\to\psi}{\psi}$\\
\DISTK & $K(p\to q)\to (Kp\to Kq)$&\GENK &$\dfrac{\phi}{K\phi}$\\
\DIST{a} & $[a](p\to q)\to ([a]p\to [a]q)$ &$\GEN{a}$ &$\dfrac{\phi}{[a]\phi} $\\
\AxTr& $Kp\to p $ &\SUB &$\dfrac{\phi(p)}{\phi(\psi)}$ \\
 \AxTrans& $Kp \to KKp$&\phantom{$\dfrac{\phi}{\Box\phi}$}\\
 \AxEuc& $\neg Kp\to K\neg Kp$&\phantom{$\dfrac{\phi}{\Box\phi}$}\\
\ZIG{a} & $K[a]p\to [a]Kp$ &&\phantom{$\dfrac{\phi}{\Box\phi}$}\\
\ZAG{a} & $\lr{a}Kp\to K[a]p$ & \phantom{$\dfrac{\phi}{\Box\phi}$}\\
\end{tabular}
\end{center}
}
\noindent where $a$ ranges over $\tA$, $p, q$ range over $\tP$. $\ZIG{\cdot}$ and $\ZAG{\cdot}$ denote the axioms of \textit{perfect recall} and \textit{no miracles} respectively (cf.\ \cite{WC12}).

Note that since we do not assume that the agent can observe the available actions, the axiom $\OBS{a}: K\lr{a}\top\lor K \neg \lr{a}\top$ in \cite{wang2012} is abandoned. Due to the same reason, the axiom of no miracles is also simplified.

We show the completeness of $\SELA$ using a more direct proof strategy compared to the one used in \cite{wang2012}.
\begin{theorem}\label{theo:completeness}
$\SELA$ is sound and strongly complete w.r.t.\ $\EAL$ on uncertainty maps.
\end{theorem}
\begin{proof}
To prove that $\SELA$ is sound on uncertainty maps, we need to show that all the axioms are valid and all the inference rules preserve validity. Since the uncertainty set in an \tUM\ denotes an equivalent class, axioms \AxTr, \AxTrans\ and \AxEuc\ are valid; due to the semantics, the validity of axioms $\ZIG{\cdot}$ and $\ZAG{\cdot}$ can be proved step by step; others can be proved as usual.

To prove that $\SELA$ is strongly complete on uncertainty maps, we only need to show that every $\SELA$-consistent set of formulas is satisfiable on some uncertainty map. The proof idea is that we construct an uncertainty map consisting of maximal $\SELA$-consistent sets (MCSs),  and then with the Lindenbaum-like lemma that every $\SELA$-consistent set of formulas can be extended in to a MCS (we omit the proof here), we only need to prove that every formula holds on the MCS to which it belongs.

Firstly, we construct a canonical Kripke model $\N^c=\lr{\calS^c,\{\calR^c_a\mid a\in\tA\},\calV^c}$ as follows:
\begin{itemize}
\item $\S^c$ is the set of all MCSs;
\item $s\R^c_a t\iff$  $\langle a\rangle\phi\in s$ for any $\phi\in t$ (equivalently $\phi\in t$ for any $[a]\phi\in s$);
\item $\V^c(p)=\{s\mid p\in s\}$.
\end{itemize}
Given $s\in\calS^c$, we define $\calU^c_s=\{u\in\calS^c\mid K\phi\in s $ iff $K\phi\in u\}$, and it is obvious that $s\in \calU^c_s$. Thus we have that for each $s\in\S^c$, $\M^c_s=\lr{\N^c,\U^c_s}$ is an uncertainty map, and $\M^c_s,s$ is a pointed uncertainty map.

Secondly, we prove the following claim.
\begin{claim}\label{claim:uncertainty}
If $s\rel{a}t$, then we have $\calU^c_s|^a=\calU^c_t$.
\end{claim}
$\subseteq$: Assuming $v\in\U^c_s|^a$, we need to show $v\in\U^c_t$, namely we need to show that $K\phi\in v\iff K\phi\in t$. Since  $v\in\U^c_s|^a$, we have that there is $u\in\U^c_s$ such that $u \R^c_a v$. If $K\phi\in t$, it follows by axiom \AxTrans\ that $KK\phi\in t$. Thus we have $\lr{a}KK\phi\in s$. By axiom \ZAG{a}, it follows that $K[a]K\phi\in s$. By $u\in\U^c_s$ and axiom \AxTr, we have $[a]K\phi\in u$. It follows by $u \R^c_a v$ that $K\phi\in v$. If $K\phi\not\in t$, we have $\neg K\phi\in t$. By axiom \AxEuc, we have $K\neg K\phi\in t$. Similarly, we have $\neg K\phi\in v$. Thus we have $K\phi\not\in v$.

$\supseteq$: Assuming $v\in \U^c_t$, we need to show $v\in\U^c_s|^a$, namely there is $u\in \U^c_s$ such that $u \R^c_a v$. Let $u^-$ be $\{K\phi\mid K\phi\in s \}\cup\{\lr{a}\psi\mid\psi\in v \}$. Then $u^-$ is consistent. For suppose not, we have $\vdash K\phi_1\wedge\dots\wedge K\phi_n\to [a]\neg\psi_1\vee\dots\vee [a]\neg\psi_k$ for some $n$ and $k$. Since $\vdash[a]\neg\psi_1\vee\dots\vee [a]\neg\psi_k\to[a](\neg\psi_1\vee\dots\vee\neg\psi_k)$, we have $\vdash K\phi_1\wedge\dots\wedge K\phi_n\to[a](\neg\psi_1\vee\dots\vee\neg\psi_k)$. By rule \GENK\ and axiom \DISTK, we have  $\vdash KK\phi_1\wedge\dots\wedge KK\phi_n\to K[a](\neg\psi_1\vee\dots\vee\neg\psi_k)$. Since $KK\phi_i\in s$ for each $1\leq i\leq n$, we have $K[a](\neg\psi_1\vee\dots\vee\neg\psi_k)\in s$. By axiom \ZIG{a}, it follows that $[a]K(\neg\psi_1\vee\dots\vee\neg\psi_k)\in s$. It follows by $s\R^c_a t$ that $K(\neg\psi_1\vee\dots\vee\neg\psi_k)\in t$. Since $v\in \U^c_t$, by axiom \AxTr, we have $\neg\psi_1\vee\dots\vee\neg\psi_k\in v$. This is contrary with $\psi_i\in v$ for each $1\leq i\leq k$. Thus $u^-$ is consistent. By Lindenbaum-like Lemma, there exists a MCS $u$ extending $u^-$. It follows by $u^-\subseteq u$ that $u\in\U^c_s$ and $u\R^c_a v$. We conclude that $v\in\U^c_s|^a$.

Finally, we will show that $\calM^c_s,s\vDash\phi$ iff $\phi\in s$. we prove it by induction on $\phi$. Please note that the `existence lemmas' (that $\neg[a]\phi\in s$ implies $\neg\phi\in t$ for some $t$ such that $s\rel{a}t$ and that $\neg K\phi\in s$ implies $\neg\phi\in s'$ for some $s'\in \U^c_s$) also hold in the model $\N^c$. We only focus on the case of $[a]\phi$. With Claim~\ref{claim:uncertainty}, it follows that $\M^c_t=\M^c_s|^a$ if $s\rel{a}t$. Then by the induction hypothesis and the existence lemmas, it is easy to show that $\M^c_s,s\vDash[a]\phi $ iff $[a]\phi\in s$.
\end{proof}

\section{An extension of EAL for conformant planning}
\label{sec.pdl}
\subsection{Epistemic PDL over uncertainty maps}
In this section we extend the language of $\EAL$ with programs in propositional dynamic logic and use this extended language to express the existence of a conformant plan.
\begin{definition}[Epistemic PDL]
The \textit{Epistemic PDL Language} (\EPDL)  is defined as follows:
\begin{center}
$\phi::=\top\mid p\mid \neg\phi\mid(\phi\wedge\phi)\mid [\pi]\phi \mid K\phi$\\
$\pi::=a\mid {?\phi}\mid  (\pi;\pi)\mid (\pi+\pi)\mid \pi^* $
\end{center}
\noindent where $p\in\tP$, $a\in \tA$. We use $\lrp{\pi}\phi$ to denote $[\pi]\phi\land \lr{\pi}\phi$,  which is logically equivalent to $[\pi]\phi\land \lr{\pi}\top$. Given a finite $\B\subseteq \tA$, we write $\B^*$ for $(\Sigma_{a\in \B}a)^*$, i.e., the iteration over the `sum' of all the action symbols in $\B$. The size of \EPDL\ formulas\slash programs is given by:
$|[\pi]\phi|=|\pi|+|\phi|$, $|a|=1$, $|\pi_1;\pi_2|=1+|\pi_1|+|\pi_2|$, $|?\phi|=|\pi^*|=1+|\phi|$, and $|\pi_1+\pi_2|=1+|\pi_1|+|\pi_2|$.
\end{definition}


Given any uncertainty map $\M=\lr{\S,\{\R_{a}\mid a\in \tA\},\V,\U}$,  any state $s\in \U$, the semantics is given by a mutual induction on $\phi$ and $\pi$ (we only show the case about $[\pi]\phi$, other cases are as in \EAL):

\[
\begin{array}{|rcl|}
\hline
\M,s\vDash [\pi]\phi &\LRA & \textrm{for all }\M',s': (\M,s)\llbracket\pi\rrbracket(\M',s')\\
					 && \text{ implies } \M',s'\vDash\phi\\
(\M,s)\llbracket a \rrbracket(\M',s') & \LRA & \M'=\M|^a \textrm{ and } s\rel{a} s' \\
(\M,s)\llbracket ?\psi\rrbracket(\M',s') & \LRA & (\M',s')=(\M,s)\textrm{ and } \M,s\vDash \psi \\
(\M,s)\llbracket \pi_1;\pi_2\rrbracket(\M',s') &\LRA & (\M,s)\llbracket \pi_1\rrbracket\circ\llbracket\pi_2\rrbracket(\M',s')\\
(\M,s)\llbracket \pi_1+\pi_2\rrbracket(\M',s') & \LRA & (\M,s)\llbracket \pi_1\rrbracket \cup \llbracket\pi_2\rrbracket(\M',s')\\
(\M,s)\llbracket \pi^*\rrbracket(\M',s') &\LRA & (\M,s)\llbracket \pi\rrbracket^\star (\M',s')\\
\hline
\end{array}
\]

\noindent where $\circ,\cup$, $^\star$ at the right-hand side denote the usual composition, union and reflexive transitive closure of binary relations respectively. Clearly this semantics coincides with the semantics of \EAL\ on \EAL\ formulas.

Note that each program $\pi$ can be viewed as a set of computation sequences, which are sequences of actions in $\tA$ and tests with $\phi\in\EPDL$:
$$\begin{array}{l}
\L(a)=\{a\}\\
\L(?\phi)=\{?\phi\}\\
\L(\pi;\pi')= \{\sigma\eta \mid \sigma \in \L(\pi) \mbox{ and } \eta \in \L(\pi')\}\\
\L({\pi} + \pi')=\L(\pi)\cup\L(\pi')\\
\L(\pi^*)=\{\epsilon\}\cup\bigcup_{n> 0}(\L(\underbrace{\pi\cdots\pi}_n)) \text{ where $\epsilon$ is the empty sequence}
\end{array}$$

Here are some valid formulas which are useful in our latter discussion:
\begin{center}
\begin{tabular}{r@{ $\lra$ }l}
$\lr{\pi;\pi'}\phi$ & $\lr{\pi}\lr{\pi'}\phi$\\
$[\pi + \pi']\phi$& $ [\pi]\phi \land [\pi']\phi$\\
$[?\psi]\phi$ & ($\psi\to \phi$)\\
\end{tabular}
\end{center}

We leave the complete axiomatization of \EPDL\ on uncertainty maps to future work.

\subsection{Conformant planning via model checking EPDL}
\begin{definition}[Conformant planning]
Given an uncertainty map $\M$, a goal formula $\phi\in\EPDL$, and a set $\B\subseteq\tA$, the conformant planning problem is to find a finite (possibly empty) sequence $\sigma=a_1a_2\cdots a_n\in\L(\B^*)$ such that for each $u\in \U_\M$ we have $\M,u\vDash \lrp{a_1}\lrp{a_2}\cdots\lrp{a_n}\phi$. The existence problem of conformant planning is to test whether such a sequence exists.
\end{definition}

Recall that $\lrp{\pi}\phi$ is the shorthand of $[\pi]\phi\land\lr{\pi}\phi$. Intuitively, we want a plan which is both executable and safe w.r.t.\ non-deterministic actions and initial uncertainty of the agent. It is crucial to observe the difference between $\lrp{a_1}\lrp{a_2}\cdots\lrp{a_n}\phi$ and $\lrp{a_1;a_2;\cdots;a_n}\phi$ by the following example:

\begin{example}
Given uncertainty map $\M$ depicted as follows, we have $\M,s_1\vDash\lrp{a;b}p$ but $\M,s_1\nvDash\lrp{a}\lrp{b}p$.
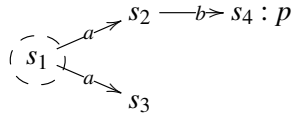
$$
\xymatrix@R-20pt{
&{s_2}\ar[r]|b&{s_4: p}\\
{s_1}\ar[ur]|a\ar[dr]|a\\
&{s_3}
\save "2,1"!C="g1"*+[F--:<+20pt>]\frm{}
\restore
}
$$
\end{example}

Given $\M$ and $\phi$, to verify whether $\sigma\in \L(\pi)$ is a conformant plan can be formulated as the model checking problem: $\M,u\vDash K \lrp{a_1}\lrp{a_2}\cdots\lrp{a_n}\phi$. On the other hand, the existence problem of a conformant plan is more complicated to formulate: it asks whether there \textit{exists} a $\sigma\in\L(\B^*)$ such that it can be verified as a conformant plan. The simple-minded attempt would be to check whether $\M,u\vDash K \lr{\B^*}\phi$ holds. Despite the $\lr{\cdot}$-vs.-$\lrp{\cdot}$ distinction, $K\lr{\B^*}\phi$ may hold on a model where the sequences to guarantee $\phi$ on different states in $\U_M$ are different, as the following example shows:

\begin{example}
Given uncertainty map $\M$ depicted as follows, let the goal formula be $p$ and $\B=\{a,b\}$. We have $\M,s_1\vDash K \lr{\B^*}p$, but there is no solution to this conformant planning problem.
$$
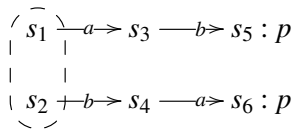
\xymatrix@R-10pt{
{s_1}\ar[r]|a&{s_3}\ar[r]|b&{s_5: p}\\
{s_2}\ar[r]|b&{s_4}\ar[r]|a&{s_6: p}
\save "1,1"."2,1"!C="g1"*+[F--:<+20pt>]\frm{}
\restore
}
$$
\end{example}

The right formula to check for the existence of a conformant plan w.r.t.\ $\B\subseteq\tA$ and $\phi\in\EPDL$ is: $$\theta_{\B,\phi}=\lr{(\Sigma_{a\in\B}(?K\lr{a}\top;a))^*}K\phi.$$ For example, if $\B=\{a_1,a_2\}$ then $\theta_{\B,\phi}=\lr{((?K\lr{a_1}\top;a_1)+(?K\lr{a_2}\top;a_2))^*}K\phi$. Intuitively, the confrmant plan consists of actions that are always executable given the uncertainty of the agent (guaranteed by the guard $K\lr{a}\top$). In the end the plan should also make sure that $\phi$ must hold given the uncertainty of the agent (guaranteed by $K\phi$). In the following, we will prove that this formula is indeed correct.

\medskip

First, we observe that the rule of substitution of equivalents is valid ($\phi(\psi\slash \chi)$ is obtained by replacing any occurrence of $\chi$ by $\psi$, similar for $\llbracket \pi(\psi/\chi) \rrbracket$):
\begin{proposition}\label{prop:substitution}
If $\vDash\psi\leftrightarrow\chi$, then:
\begin{itemize}
\item[(1)] $\vDash\phi\leftrightarrow\phi(\psi/\chi)$;
\item[(2)] $\llbracket \pi \rrbracket=\llbracket \pi(\psi/\chi) \rrbracket$.
\end{itemize}
\end{proposition}

\begin{proposition}\label{prop:basicCase}
	$\vDash K\lrp{a}\phi\leftrightarrow\lr{?K\lr{a}\top;a}K\phi$
\end{proposition}
\begin{proof}
	Since $\vDash K\lrp{a}\phi\leftrightarrow (K[a]\phi\wedge K\lr{a}\phi)$ and $\vDash (K\lr{a}\top\wedge\lr{a}K\phi)\leftrightarrow \lr{?K\lr{a}\top;a}K\phi$, we only need to show that $\vDash (K[a]\phi\wedge K\lr{a}\phi)\leftrightarrow   (K\lr{a}\top\wedge\lr{a}K\phi)$.
		
	Left to right:\\
	\indent (L1) $\vDash K[a]\phi\to [a]K\phi$, by validity of Axiom \ZIG{a}\\
	\indent (L2) $\vDash K\lr{a}\phi\to \lr{a}\top\wedge K\lr{a}\top$, by semantics\\
	\indent (L3) $\vDash \lr{a}\top\wedge[a]K\phi\to \lr{a}K\phi$, by semantics\\
	\indent (L4) $\vDash K[a]\phi\wedge K\lr{a}\phi\to K\lr{a}\top\wedge\lr{a} K\phi$, by (L1)-(L3)
	
	Right to left:\\
	\indent (R1) $\vDash \lr{a}K\phi\to K[a]\phi$, by validity of Axiom \ZAG{a}\\
	\indent (R2) $\vDash K[a]\phi\wedge K\lr{a}\top\to K\lr{a}\phi$, by semantics\\
	\indent (R3) $\vDash K\lr{a}\top\wedge\lr{a} K\phi\to K[a]\phi\wedge K\lr{a}\phi$, by R(1)-R(2)	
\end{proof}
\begin{lemma}\label{prop:exist_know}
For any $a_1a_2\cdots a_n\in\L(\tA^*)$:
$$\vDash K\lrp{a_1}\lrp{a_2}\cdots\lrp{a_n}\phi \leftrightarrow  \lr{?K\lr{a_1}\top;a_1 ;\dots ; ?K\lr{a_n}\top;a_n}K\phi$$
\end{lemma}
\begin{proof}
It is trivial when $n=0$ (i.e., the sequence is $\epsilon$), since the claim then boils down to $K\phi\lra K\phi$.
We prove the non-trivial cases by induction on $n\geq 1$. When $n=1$, it follows from Proposition~\ref{prop:basicCase}.
Now, as the induction hypothesis, we assume that:
$$\vDash K\lrp{a_1}\lrp{a_2}\cdots\lrp{a_k}\phi \leftrightarrow  \lr{?K\lr{a_1}\top;a_1 ;\dots ; ?K\lr{a_k}\top;a_k}K\phi.$$
We need to show:
\begin{align*}
\vDash & K\lrp{a_1}\lrp{a_2}\cdots\lrp{a_{k+1}}\phi \leftrightarrow \lr{?K\lr{a_1}\top;a_1 ;\dots ; ?K\lr{a_{k+1}}\top;a_{k+1}}K\phi.
\end{align*}
By IH,
\begin{align*}
\vDash & K\lrp{a_1}\lrp{a_2}\cdots\lrp{a_{k+1}}\phi \leftrightarrow \lr{?K\lr{a_1}\top;a_1 ;\dots ; ?K\lr{a_k}\top;a_k}K\lrp{a_{k+1}}\phi. \qquad (1)
\end{align*}
Due to Propositions \ref{prop:substitution} and \ref{prop:basicCase}, we have: \begin{align*}
\vDash &\lr{?K\lr{a_1}\top;a_1 ;\dots ; ?K\lr{a_k}\top;a_k}K\lrp{a_{k+1}}\phi \leftrightarrow \lr{?K\lr{a_1}\top;a_1 ;\dots ; ?K\lr{a_n}\top;a_k}\lr{?K\lr{a_{k+1}}\top;a_{k+1}}K\phi. \ (2)
\end{align*} The conclusion is immediate by combining (1) and (2).
\end{proof}

The following theorem follows from the above lemma.
\begin{theorem}\label{theo:conformant_planning}
	Given a pointed uncertainty map $\M,s$, an \EPDL\ formula $\phi$ and a set $\B\subseteq\tA$, the following two are equivalent:
\begin{itemize}
\item[(1)] There is a $\sigma=a_1\dots a_n\in\L(\B^*)$ such that $\calM,s\vDash K \lrp{a_1}\lrp{a_2}\cdots\lrp{a_n}\phi$;
\item[(2)] $\calM,s\vDash \lr{(\Sigma_{a\in\B}(?K\lr{a}\top;a))^*}K\phi$.
\end{itemize}
\end{theorem}
%
\medskip

We would like to emphasise that the $K$ operator right before $\phi$ in the definition of $\theta_{\B,\phi}$ cannot be omitted, as demonstrated by the following example:

\begin{example}
Given uncertainty map $\M$ depicted as follows, let the goal formula be $p$. As we can see, there is no solution to this conformant planning problem. Indeed $\calM,s_1\nvDash \lr{(\Sigma_{a\in\B}(?K\lr{a}\top;a))^*}K p$ with $\B=\{a,b\}$, but we could have $\calM,s_1\vDash \lr{(\Sigma_{a\in\B}(?K\lr{a}\top;a))^*}p$.
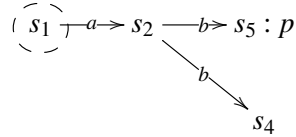
$$
\xymatrix{
{s_1}\ar[r]|a&{s_2}\ar[r]|b\ar[dr]|b&{s_5: p}\\
&&{s_4}
\save "1,1"!C="g1"*+[F--:<+20pt>]\frm{}
\restore
}
$$
\end{example}

We close this section with an example about planning with both positive and negative epistemic goals (the agent should know something, but not too much).
\begin{example}
Given uncertainty map $\M$ depicted as follows, let the goal be $Kp$ then both $a$ and $b$ are conformant plans. If the goal is $Kp\wedge \neg K q$, only $a$ is a good plan.
$$
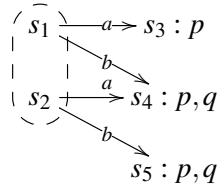
\xymatrix@R-10pt{
{s_1}\ar[r]|a\ar[dr]|b&{s_3:p}\\
{s_2}\ar[r]^a\ar[dr]|b&{s_4:p,q}\\
&{s_5: p,q}
\save "1,1"."2,1"!C="g1"*+[F--:<+20pt>]\frm{}
\restore
}
$$
\end{example}

\section{Model checking EPDL: complexity and algorithms}
\label{sec.mc}
%
%
%

In this section, we first focus on the model checking problem of the following star-free fragment of \EPDL\ (call it \EPDLsf):
\begin{center}
		$\phi::=\top\mid p\mid \neg\phi\mid(\phi\wedge\phi)\mid [\pi]\phi \mid K\phi$\\
		$\pi::=a\mid {?\phi}\mid  (\pi;\pi)\mid (\pi+\pi) $\end{center}
We will show that model checking \EPDLsf\ is \PSPACE-complete. In particular, the upper bound is shown by making use of an alternative context-dependent semantics. Then we give an \EXPTIME\ algorithm for the model checking problem of the full \EPDL\ inspired by another alternative semantics based on 2-dimensional models. Finally we give a \PSPACE\ algorithm for the conformant planning problem in \EPDL.  Note that throughout this section, we focus on uncertainty maps with finitely many states and assume $\R_a=\emptyset$ for \textit{co-finitely} many $a\in \tA$.

\subsection{Complexity of model checking EPDL$^-$}
\subsubsection{Lower Bound}

To show the \PSPACE\ lower bound, we provide a polynomial reduction of QBF (\textit{quantified Boolean formula}) truth testing to the model checking problem of $\EPDLsf$. Note that to determine whether a given QBF (even in prenex normal form based on a conjunctive normal form) is true or not is known to be \PSPACE-complete~\cite{stockmeyer1973}. Our method is inspired by \cite{phs02} which discusses the complexity of model checking temporal logics with past operators. Surprisingly, we can use the uncertainty sets to encode the `past' and use the dual of the knowledge operator to `go back' to the past. This intuitive idea will become more clear in the proof.

\bigskip

QBF formulas are $Q_1 x_1 Q_2x_2\dots Q_nx_n\phi(x_1,\dots,x_n)$ where:
\begin{itemize}
\item For $1\leq n\leq n, Q_i$ is $\exists$ if $i$ is odd, and $Q_i$ is $\forall$ if $i$ is even.
\item $\phi$ is a propositional formula in CNF based on variables $x_1,\dots,x_n$,
\end{itemize}

For each such QBF $\alpha$ with $n$ variables, we need to find a pointed model $\M_n,x_0$ and a formula $\theta_\alpha$ such that $\alpha$ is true iff $\M_n,x_0\vDash\theta_\alpha$. The model $\M_n$ is defined below.

\begin{definition}
Let $\tA=\{a_i,\bar{a}_i\mid i\geq 1  \}$ and $\tP=\{p_k,q_k\mid k\geq 1 \}$, the uncertainty map $\M_n=\langle \calS, \{\calR_{a}\mid a\in \tA\}, \calV, \calU \rangle$ is defined as:
\begin{itemize}
\item $\S=\{x_0\}\cup\{x_i\mid 1\leq i \leq n\}\cup \{\bar{x}_i\mid 1\leq i \leq n\}$
\item $\V(x_0)=\emptyset$, and $\V(x_i)=\{p_i\}, \V(\bar{x}_i)=\{q_i\}$ for $1\leq i\leq n$.
\item $\rel{a_i}=\{(s,s)\mid s\in\S \}\cup\{(x_{i-1},x_i),(\bar{x}_{i-1},x_i) \}$
\item $\rel{\bar{a}_i}=\{(s,s)\mid s\in\S \}\cup\{(x_{i-1},\bar{x}_i),(\bar{x}_{i-1},\bar{x}_i) \}$
\item $\U=\{x_0\}$
\end{itemize}
\end{definition}
 $|\M_n|$ is linear in $n$ and can be depicted as the following:

$$
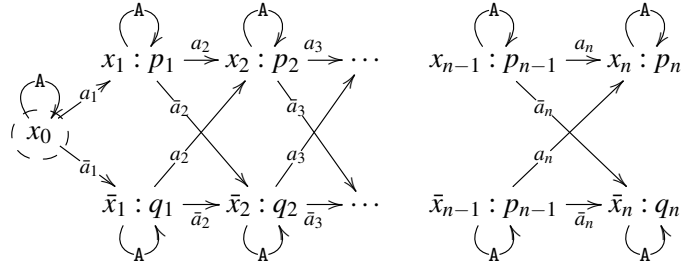
\xymatrix@C-=0.45cm@R-10pt{
                               &x_1:p_1\ar@(lu,ru)[]|\tA\ar[r]^{a_2}\ar[ddr]|(.33){\bar{a}_2}&  x_2:p_2\ar@(lu,ru)[]|\tA\ar[r]^{a_3}\ar[ddr]|(.33){\bar{a}_3}&\cdots& x_{n-1}:p_{n-1}\ar@(lu,ru)[]|\tA\ar[r]^(.60){a_n}\ar[ddr]|(.33){\bar{a}_n}& x_n:p_n\ar@(lu,ru)[]|\tA\\
x_0\ar@(lu,ru)[]|\tA\ar[ur]|{a_1}\ar[dr]|{\bar{a}_1}&   \\
                               &\bar{x}_1:q_1\ar@(ld,rd)[]|\tA\ar[r]_{\bar{a}_2}\ar[uur]|(.33){a_2} &\bar{x}_2:q_2\ar@(ld,rd)[]|\tA\ar[r]_{\bar{a}_3}\ar[uur]|(.33){a_3} &\cdots &\bar{x}_{n-1}:p_{n-1}\ar@(ld,rd)[]|\tA\ar[r]_(.60){\bar{a}_n}\ar[uur]|(.33){a_n}&\bar{x}_n:q_n\ar@(ld,rd)[]|\tA
\save "2,1"!C="g1"*+[F--:<+20pt>]\frm{}
\restore
}
$$

Given $\alpha=Q_1 x_1 Q_2x_2\dots Q_nx_n\phi(x_1,\dots,x_n)$, the formula $\theta_\alpha$ is defined as $$\qt_1\cdots\qt_n\psi(\hat{K}p_1,\cdots,\hat{K}p_n,\hat{K}q_1,\cdots, \hat{K}q_n)$$ where $\qt_i$ is $\lr{(a_i+\bar{a}_i);?(p_i\vee q_i)}$ if $i$ is odd and $\qt_i$ is $[(a_i+\bar{a}_i);?(p_i\vee q_i)]$ if $i$ is even, and $\psi$ is obtained from $\phi(x_1,\dots,x_n)$ by replacing each $x_i$ with $\hat{K}p_i$ and $\neg x_i$ with $\hat{K}q_i$.

To ease the latter proof, we first define the valuation tree below.

\begin{definition}[V-tree]
A V-tree $\tau$ is a rooted tree such that 1) each node is $0$ or $1$ (except the root $\epsilon$); 2) each internal node in an even level has only one successor; 3) each internal node in an odd level has two successors:  one is $0$ and the other one is $1$; 4) each edge to node $0$ of level $i$ is labelled $\bar{a}_i$; 5) each edge to node $1$ of level $i$ is labelled $a_i$. Given a V-tree with depth $n$, a path $\sigma$ is a sequence of $A_1\dots A_n$ where $A_i=a_i$ or $A_i=\bar{a}_i$. A path $\sigma$ can also be seen as a valuation assignment for $x_1,\dots,x_n$ with the convention that $\sigma(x_i)=1$ if $a_i$ occurs in $\sigma$ and $\sigma(x_i)=0$ if $\bar{a}_i$ occurs in $\sigma$. Let $path(\tau)$ be the set of all paths of $\tau$.
\end{definition}
As an example, a V-tree $\tau$ can be depicted as below:
$$
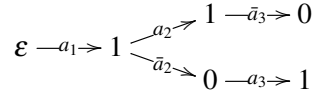
\xymatrix@u@C-25pt{
&\epsilon\ar[d]|{a_1}&\\
&1\ar[dl]|{\bar{a}_2}\ar[dr]|{a_2}&\\
0\ar[d]|{a_3}&&1\ar[d]|{\bar{a}_3}\\
1&&0
}
$$

It is not hard to see the following:
\begin{proposition}\label{prop:qbooleanf}
For each $1\leq i\leq n$, we have: 
$\alpha=Q_1x_1\dots Q_ix_i Q_{i+1}x_{i+1}\dots Q_nx_n\phi$ is true  iff 
there exists a V-tree $\tau$ with depth $i$ such that  for each $\sigma\in path(\tau)$ $\sigma(Q_{i+1}x_{i+1}\dots Q_nx_n\phi)=1$ ($\sigma$ as a valuation).
\end{proposition}
Now let us see the update result of running a path $\sigma\in path(\tau)$ on $\M_n$. Due to  the lack of space, we omit the proofs of the following two propositions.
\begin{proposition}\label{prop:vuncertaintyset}
Given $\M_n$, let $\sigma=A_1\dots A_i$ $(1\leq i\leq n)$ be a sequence of actions such that $A_k=a_k$ or $A_k=\bar{a}_k$ for each $1\leq k\leq i$, then we have $\U|^\sigma=\{x_0,X_1,\dots,X_i \}$ where $X_k=x_k$ if $A_k=a_k$ else $X_k=\bar{x}_k$ for each $1\leq k\leq i$.
\end{proposition}

Given $\sigma=A_1\dots A_n$ where $A_i$ is $a_i$ or $\bar{a}_i$ for each $1\leq i\leq n$, let $g(\sigma)=x_n$ if $A_n=a_n$ and $g(\sigma)=\bar{x}_n$ if $A_n=\bar{a}_n$. By Proposition \ref{prop:vuncertaintyset}, we always have $g(\sigma)\in \U_{\M_k}|^\sigma$ with $k>n$. Thus given $\M_k$ and $\sigma=A_1\dots A_n$ and $k>n$, $\M_k|^\sigma,g(\sigma)$ is a pointed uncertainty map.

\begin{proposition}\label{prop:vumodel}
For each $1\leq i\leq n$, we have $\M_k,x_0\vDash \qt_1\dots\qt_i\qt_{i+1}\dots\qt_n\psi$ iff there exists a V-tree $\tau$ with depth $i$ such that $\M_k|^\sigma,g(\sigma)\vDash \qt_{i+1}\dots\qt_n\psi$ for each $\sigma\in path(\tau)$, where $k>n$ and $g(\sigma)$ is the state corresponds to the last edge of $\sigma$, e.g., $g(a_1\bar{a}_2)=\bar{x}_2$.
\end{proposition}

\begin{theorem}The following two are equivalent:
\begin{itemize}
\item $\alpha= Q_1 x_1 Q_2x_2\dots Q_nx_n\phi(x_1,\dots,x_n)$ is true
\item $\M_n,x_0\vDash \qt_1\cdots\qt_n\psi(\hat{K}p_1\cdots\hat{K}p_n,\hat{K}q_1\cdots\hat{K}q_n)$ in which $\psi$ is obtained from $\phi$ by replacing each $x_i$ with $\hat{K}p_i$ and $\neg x_i$ with $\hat{K}q_i$.
\end{itemize}
\end{theorem}
\begin{proof}
By Propositions \ref{prop:qbooleanf} and \ref{prop:vumodel}, we only need to show that given V-tree $\tau$ with depth $n$, $\sigma(\phi)=1$ if and only if $\M_n|^\sigma,g(\sigma)\vDash \psi$ for each $\sigma\in path(\tau)$.  Since $\phi$ is in CNF, $\psi$ is also in CNF-like form obtained by replacing each $x_i$ with $\hat{K}p_i$ and each $\neg x_i$ with $\hat{K}q_i$ for $1\leq i\leq n$. Thus we only need to show that $\sigma(x_i)=1$ iff $\M_n|^\sigma,g(\sigma)\vDash \hat{K}p_i$ and $\sigma(\neg x_i)=1$ iff $\M_n|^\sigma,g(\sigma)\vDash \hat{K}q_i$. Since $\sigma(x_i)=1$ iff $\sigma(\neg x_i)=0$, we only need to show that  $\sigma(x_i)=1$ iff $\M_n|^\sigma,g(\sigma)\vDash \hat{K}p_i$ and $\M_n|^\sigma,g(\sigma)\vDash \hat{K}p_i$ iff $\M_n|^\sigma,g(\sigma)\vDash\neg \hat{K}q_i$. By the definition of $\tau$, we know that $\sigma=A_1\dots A_n$ where $A_i$ is $a_i$ or $\bar{a}_i$ for each $1\leq i\leq n$.

Firstly, we will show that  $\M_n|^\sigma,g(\sigma)\vDash \hat{K}p_i$ if and only if $\M_n|^\sigma,g(\sigma)\vDash\neg \hat{K}q_i$. To verify the right-to-left direction, if $\M_n|^\sigma,g(\sigma)\vDash \hat{K}p_i$, it follows by the definition of $\M_n$ that $x_i\in \U|^\sigma$. Then it must be the case that $a_i$ occurs in  $\sigma$. Suppose not, $\bar{a}_i$ occurs in $\sigma$. It follows by Proposition \ref{prop:vuncertaintyset}, $\U|^\sigma=\{x_0,X_1,\dots,X_{i-1},\bar{x}_i,X_{i+1},\dots,X_n\}$. This is contrary with $x_i\in \U|^\sigma$. Thus it must be that $a_i$ occurs in  $\sigma$. It follows by Proposition \ref{prop:vuncertaintyset} that $\U|^\sigma=\{x_0,X_1,\dots,X_{i-1},x_i,X_{i+1},\dots,X_n\}$. Thus $\bar{x}_i\not\in \U|^\sigma$. By the definition of $\M_n$ and the semantics, we have $\M_n|^\sigma,g(\sigma)\vDash\neg \hat{K}q_i$. To verify the left-to-right direction, $\M_n|^\sigma,g(\sigma)\vDash\neg \hat{K}q_i$ implies that $\bar{x}_i\not\in \U|^\sigma$. For the similar reason as above, it must be the case that $\bar{a}_i$ does not occur in $\sigma$. Thus we have that $a_i$ occurs in $\sigma$. It follows by Proposition \ref{prop:vuncertaintyset} that $x_i\in \U|^\sigma$. Thus we have $\M_n|^\sigma,g(\sigma)\vDash \hat{K}p_i$.

Next we will show that $\sigma(x_i)=1$ iff $\M_n|^\sigma,g(\sigma)\vDash \hat{K}p_i$. To verify the right-to-left direction, $\sigma(x_i)=1$ implies that $A_i=a_i$. It follows by Proposition \ref{prop:vuncertaintyset} that $x_i\in\U|^\sigma$. Thus we have $\M_n|^\sigma,g(\sigma)\vDash \hat{K}p_i$. To verify the left-to-right direction, we will show that $\sigma(x_i)=0$ implies $\M_n|^\sigma,g(\sigma)\vDash \hat{K}q_i$. It follows by the definition of $\sigma(x_i)=0$ that $A_i=\bar{a}_i$. It follows by Proposition \ref{prop:vuncertaintyset} that $\bar{x}_i\in\U|^\sigma$. Thus we have $\M_n|^\sigma,g(\sigma)\vDash \hat{K}q_i$.
\end{proof}

This gives us the desired lower bound:
\begin{theorem}\label{thm.lowerb}
The model checking problem for $\EPDL^-$ is \PSPACE-hard.
\end{theorem}


\subsubsection{Upper Bound}\label{sec:epdl_minus}

In this section we give a non-trivial model checking algorithm for $\EPDLsf$ inspired by an equivalent semantics.

As we mentioned earlier, the semantics of \EPDL\ is `context-dependent': reaching the same state through different paths may affect the truth value of an epistemic subformula. This means that the usual global model checking algorithm for modal logics may not work here. In order to establish the upper bound, we first give the following equivalent semantics to \EPDLsf\ which makes the context dependency explicit in order to facilitate a local model checking algorithm. The idea is to keep the model intact but record the scope of action modalities in order to compute the right uncertainty set for epistemic subformulas. Similar idea appeared in \cite{WC12} to give an alternative semantics of public announcement logic.

\begin{definition}\label{semanticsth}
Given an uncertainty map $\calM=\langle \calS, \{\calR_{a}\mid a\in \tA\}, \calV, \calU \rangle$ and any state  $s\in \calS$, the satisfaction relation $\Vdash$ is defined using the auxiliary satisfaction relation $\Vdash_\sigma$ and auxiliary relation $\rel{\omega_\sigma}$, where $\sigma$ is a finite (possibly empty) sequence of actions in $\tA$:
	\begin{center}
		$
		\begin{array}{|lcl|}
		\hline
		\M,s\Vdash \phi  & \Leftrightarrow & \M,s\Vdash_\epsilon \phi  \\
		\M,s\Vdash_\sigma \top  & \Leftrightarrow &   \textrm{ always }\\
		\M,s\Vdash_\sigma p  & \Leftrightarrow & p \in \calV(s) \\
		
		\M,s\Vdash_\sigma \neg\phi &\Leftrightarrow& \M,s\nVdash_\sigma \phi \\
		
		\M,s\Vdash_\sigma \phi\land \psi &\Leftrightarrow& \M,s\Vdash_\sigma \phi \textrm{ and } \M,s\Vdash_\sigma \psi \\
		\M,s\Vdash_\sigma K\phi&\Leftrightarrow&\text{for all }v\in \calU|^{\sigma}: \M,v\Vdash_{\sigma}\phi \\
		\M,s\Vdash_\sigma \langle \pi\rangle\phi &\Leftrightarrow& \text{ there exists } \omega \in \L(\pi) \tr{ and }t\in\calS\\
				&& \text{ such that } s\rel{\omega_\sigma} t \text{ and } \M,t\Vdash_{\sigma r(\omega)}\phi\\
		s\rel{\epsilon_\sigma} t&\Leftrightarrow&s=t\\
		s\rel{(a\omega')_\sigma} t&\Leftrightarrow&\tr{ there exists }s'\tr{ such that } s\rel{a}s'\tr{ and }		s'\rel{\omega'_{(\sigma a)}}t\\
		s\rel{(?\phi\omega')_{\sigma}}t&\Leftrightarrow&\M,s\Vdash_{\sigma}\phi\tr{ and } s\rel{\omega'_{\sigma}} t\\
		\hline
		\end{array}
		$
	\end{center}
	where $r(\omega)$ is the sequence of actions obtained by eliminating all the tests in $\omega$. 
\end{definition}
Note that $\omega$ in the above definition  is a computation sequence, i.e., a finite sequence of actions and \EPDLsf-tests, while $\sigma$ is a test-free sequence of actions.
\medskip

The following can be proved by induction on $\eta$:
\begin{proposition}\label{pro:sequece_action_in_context}
Given an uncertainty map $\calM$ and sequences of actions and tests $\eta,\omega,\omega'$ such that $\eta=\omega \omega'$, we have $(s,t)\in\rel{\eta_\sigma}$ iff $(s,t)\in\rel{\omega_\sigma}\circ\rel{\omega'_{\sigma r(\omega)}}$ for any sequence of actions $\sigma$.
\end{proposition}
\begin{proof}
	We prove it by induction on $|\eta|$. If $|\eta|\leq 2$, it is obvious by the definition. If $|\eta|>2$, there are two cases, that is, $\eta=a \eta'$ or $\eta=?\phi\eta'$.
	
	Case $\eta=a \eta':$ We have $\omega=a\omega''$ for some initial segment $\omega''$ of $\eta'$, and $(s,t)\in\rel{(a\eta')_\sigma}$ iff there exists $s'$ such that $s\rel{a}s'$ and $(s',t)\in\rel{\eta'_{\sigma a}}$. By IH, we have $\rel{\eta'_{\sigma a}}=\rel{\omega''_{\sigma a}}\circ\rel{\omega'_{\sigma a r(\omega'')}}$. Thus we have $(s',t)\in\rel{\eta'_{\sigma a}}$ iff there exists $t' $ such that $(s',t')\in \rel{\omega''_{\sigma a}}$ and $(t',t)\in \rel{\omega'_{\sigma a r(\omega'')}}$. By definition, we have that $s\rel{a} s'$ and $(s',t')\in \rel{\omega''_{\sigma a}}$ iff $(s,t')\in\rel{a\omega''_{\sigma}}$. Thus we have $(s,t)\in \rel{a\omega''_\sigma}\circ\rel{\omega'_{\sigma a r(\omega'')}}$, namely $(s,t)\in \rel{\omega_\sigma}\circ\rel{\omega'_{\sigma r(\omega)}}$.
	
	Case $\eta=?\phi\eta':$ We have $\omega=?\phi\omega''$ for some initial segment $\omega''$ of $\eta'$, and $(s,t)\in\rel{(?\phi\eta')_\sigma}$ iff $\calM,s\Vdash_\sigma\phi$ and $s\rel{\eta'_\sigma}t$. By IH, we have $s\rel{\eta'_\sigma}t$ iff $(s,t)\in \rel{\omega''_{\sigma}}\circ\rel{\omega'_{\sigma r(\omega'')}}$. Thus we have there exists $s'$ such that $(s,s')\in \rel{\omega''_\sigma}$ and $(s',t)\in\rel{\omega'_{\sigma r(\omega'')}}$. This follows that $(s,s')\in \rel{(?\phi\omega'')_\sigma}$, and $(s,t)\in \rel{(?\phi\omega'')_\sigma}\circ \rel{\omega'_{\sigma r(?\phi\omega'')}}$, namely $(s,t)\in \rel{\omega_\sigma}\circ\rel{\omega'_{\sigma r(\omega)}}$.
\end{proof}

In the following we show that $\Vdash$ coincides with $\vDash$.
\begin{theorem}
	Given an uncertainty map $\calM$ and an action sequence $\sigma$, if $\calU|^\sigma\neq \emptyset$, we have that for each $s\in\calU|^\sigma$,
	\begin{itemize}
		\item[(i)] $\calM|^\sigma,s\llrr{\pi}\calM',s'$ iff there exists $\omega\in\L(\pi)$ such that $\calM'=\calM|^{\sigma r(\omega)}$ and $s\rel{\omega_\sigma}s'$,
		\item[(ii)] $\calM|^\sigma,s\vDash\phi$ iff $\calM,s\Vdash_\sigma\phi$.
	\end{itemize}
\end{theorem}
\begin{proof}
	The proof is by simultaneous induction on $\pi$ and $\phi$ (due to the test actions). For (i), we will only focus on the case of $\pi_1;\pi_2$; the other cases are straightforward.
	
	Case $\pi_1;\pi_2$: We only show the direction from left to right; the other direction is similar. It follows by assumption that there is pointed uncertainty map $\calN,t$ such that $\calM|^\sigma,s\llrr{\pi_1}\calN,t$ and $\calN,t\llrr{\pi_2}\calM',s'$. By IH, we have that there exists $\omega\in\L(\pi_1)$ such that $\calN=\calM|^{\sigma r(\omega)}$ and $s\rel{\omega_\sigma}t$. Since $\calN,t$ is a pointed uncertainty map and $\calN=\calM|^{\sigma r(\omega)}$, we have $t\in \calU|^{\sigma r(\omega)}$. By IH and $\calM|^{\sigma r(\omega)},t\llrr{\pi_2}\calM',s'$, we have that there exists $\omega'\in \L(\pi_2)$ such that $\calM|^{\sigma r(\omega) r(\omega')}=\calM|^{\sigma r(\omega\omega')}=\calM'$ and $t\rel{\omega'_{\sigma r(\omega)}}s'$. By Proposition~\ref{pro:sequece_action_in_context}, it follows that $\omega\omega'\in\L(\pi_1;\pi_2)$ and $s\rel{(\omega\omega')_\sigma}s'$.

	For (ii), we will focus on the case of $\lr{\pi}\phi$; the other cases are straightforward.
	
	Case $\lr{\pi}\phi$: We have $\calM|^\sigma,s\vDash\lr{\pi}\phi$ if and only if there is pointed uncertainty map $\calM',s'$ such that $\calM|^\sigma,s\llrr{\pi}\calM',s'$ and $\calM',s'\vDash\phi$. By (i), it follows that $\calM|^\sigma,s\llrr{\pi}\calM',s'$ iff there exists $\omega\in\L(\pi)$ such that $\calM'=\calM|^{\sigma r(\omega)}$ and $s\rel{\omega_\sigma}s'$. By IH, it follows that $\calM|^{\sigma r(\omega)},s'\vDash\phi$ iff $\calM,s'\Vdash_{\sigma r(\omega)}\phi$. Thus we have $\calM,s\Vdash\lr{\pi}\phi$.
\end{proof}

Let $\sigma$ be $\epsilon$, we have the equivalence of $\Vdash$ and $\vDash$.
\begin{corollary}\label{corollary:correspondence_context_dep}
	Given pointed uncertainty map $\calM,s$, we have $\calM,s\vDash\phi$ iff $\calM,s\Vdash\phi$ for each $\phi\in\EPDL^-$.
\end{corollary}

This alternative semantics induces a natural algorithm to compute the truth value of an $\EPDLsf$ formula w.r.t.\ to a pointed uncertainty map. The idea is to recursively call a function $MC(\M,s,\sigma, \phi)$ which returns the truth value of a subformula $\phi$ on state $s$ given the context of $\sigma$ while keeping $\M$ intact. Note that, we do not need to compute all the $MC(\M,s,\sigma, \phi)$ for each $\sigma$ and each subformula $\phi$. The only tricky part comes when evaluating $\lr{\pi}\phi$ formulas since it is too space consuming to compute the whole set of $\L(\pi)$ in the search of the right $\omega$. Instead, we can generate one by one in some lexicographical order all the possible sequences up to a bound based on the atomic actions and tests occurring in the formula, and then test whether it belongs to the program $\pi$. Note that in this way, we can use the space repeatedly, and the membership testing of $\L(\pi)$ is not expensive (\textsc{NLOGspace}-complete according to \cite{JR91}).

In the appendix we present three algorithms based on matrix representation of the model: Algorithm~\ref{CCS} computes the uncertainty set $\U|^\sigma$; Algorithm \ref{CPW} computes $\rel{w_{\sigma}}$ and Algorithm \ref{MCAG} is the main model checking algorithm. Note that Algorithms \ref{CPW} and \ref{MCAG} involve mutual recursion of each other due to the tests in programs. However, the depth of the recursion is bounded by the length of the formula, and for each call polynomial space suffices. The detailed algorithms and complexity analysis can be found in the appendix. It is not hard to show the following (based on Theorem \ref{thm.lowerb})

\begin{theorem}[Upper bound]
The model checking problem of $\EPDL^-$ is in \PSPACE. Thus it is \PSPACE-complete.	
\end{theorem}

\subsection{Upper Bounds for model checking EPDL}

In this section, we give an \EXPTIME\  model checking method for the full \EPDL\ via model checking \EPDL\ over two-dimensional models with both epistemic and action relations. Let us first define such models.
\begin{definition}[Epistemic Temporal Structure]
	An \textit{Epistemic Temporal Structure} (ETS) is a Kripke model with both epistemic and action relations. Formally, an ETS\ model $\etsM$ is a tuple $\lr{\S,\{\R_a\mid a\in\tA\},\sim,\V}$, where $\R_a$ is a binary relation on $\S$, $\sim$ is an equivalence relation on $\S$ and $\V:\S\to 2^\tP$ is a valuation function. 
\end{definition}

Now we define an alternative semantics of \EPDL\ over ETSs.\footnote{Here we abuse the notation $\Vdash$ to denote the new semantics. Note that it is different from the alternative semantics in the previous section. }

\begin{definition}[ETS Semantics]\label{def.ets}
	Given any ETS \\ model $\etsM=\langle \S, \{\R_{a}\mid a\in \tA\}, \sim,\V \rangle$ and any state $s\in \S$, the satisfaction relation for \EPDL\ formulas is defined
	as follows (the Boolean cases are as in the standard modal logic):
	$$\begin{array}{|lcl|}
	\hline
	\etsM,s\Vdash K\phi &\LRA& \forall u\in\S: s\sim u \text{ implies }\etsM,u\Vdash \phi\\
	\etsM,s\Vdash [\pi]\phi &\LRA& \forall t\in S: s\rel{\pi}t \text{ implies } \etsM,t\Vdash \phi\\
	\rel{a}&=& \R_a\\
	\rel{?\phi}&=&\{(s,s)\mid \etsM,s\Vdash\phi \}\\
	\xrel{\pi_1;\pi_2}&=&\rel{\pi_1}\circ\rel{\pi_2}\\
	\xrel{\pi_1+\pi_2}&=&\rel{\pi_1}\cup\rel{\pi_2}\\
	\rel{\pi^{*}}&=&(\rel{\pi})^\star\\
	\hline
	\end{array}
	$$
	
\noindent 	where $\circ,\cup$, $^\star$ at right-hand side denote the usual composition, union and reflexive transitive closure of binary relations respectively.
\end{definition}

We can turn a Kripke model without the epistemic relation into an ETS model by essentially considering all the possible uncertainty sets.

\begin{definition}
	Given any Kripke model $\calM=\lr{\calS,\{\calR_a\mid a\in\tA \},\calV}$, we define the ETS model $\calM^\bullet
    $ as follows:
	\[
	\begin{array}{lcl}
	\calS^\bullet&=&\{s_\Gamma\mid s\in \calS,\Gamma\in 2^\calS,s\in\Gamma \}\\
	\calR^\bullet_a&=&\{(s_\Gamma,t_\Delta)\mid s\rel{a}t,\Delta=\Gamma|^a \}\\
	\sim^\bullet&=&\{(s_\Gamma,t_\Delta)\mid \Gamma=\Delta \}\\
	\calV^\bullet(s_\Gamma)&=&\calV(s)
	\end{array}
	\]
	
	\noindent where $\Gamma|^a=\{t\in\calS\mid \exists s\in\Gamma$ such that $s\rel{a}t \}$. For any Kripke model $\M$ and any $\Gamma\in 2^\calS\backslash\{\emptyset\}$, let $\calM^\Gamma$ be the uncertainty map $\lr{\M,\Gamma}$.
\end{definition}

Note that each $s_\Gamma$ can be viewed as an uncertainty set ($\Gamma$) with a designated state ($s$), and the definition of $\R_a$ captures the update in the $\vDash$ semantics of \EPDL, and $\M^\bullet$ unravels all the updates in a whole picture. Note that the size of $\M^\bullet$ is $|\S|\cdot 2^{|\S|-1}$ where $\S$ is the set of states of $\M$.


Now we can show that $\vDash$ and $\Vdash$ coincide w.r.t.\ uncertainty map $\M^\Gamma$ and ETS model $\M^\bullet$ (the proofs are omitted due to the lack of space).
\begin{proposition}
	Given any map $\calM$, we have
	\begin{itemize}
		\item[(i)]  $\calM^\Gamma,s\llrr{\pi}\calM^\Delta,t$ iff $s_\Gamma\rel{\pi}t_{\Delta}$ in $\M^\bullet$;\footnote{Cf. the definition of $\rel{\pi}$ in Def.\ \ref{def.ets}. }
		\item[(ii)] $\calM^\Gamma,s\vDash\phi$ iff $\calM^\bullet,s_\Gamma\Vdash\phi$.
	\end{itemize}
\end{proposition}
\begin{corollary}\label{corollary:ETS_EPDL}
	Given an uncertainty map $\calM=\lr{\N,\U}$ and $s\in\calU$, we have $\calM,s\vDash\phi$ iff $\calN^\bullet,s_\calU\Vdash\phi$.
\end{corollary}
Based on the above corollary we can have a model checking method via model checking \EPDL\ over ETS models.
\begin{proposition}
	The model checking problem of \EPDL\ on uncertainty maps is in \EXPTIME.
\end{proposition}
\begin{proof}
	Given an uncertainty map $\calM=\lr{\N, \U}$, the construction of ETS $\calN^\bullet$ can be done in exponential time in the size of $\N$ due to the fact that there are at most $|\N|$ $a$-successors $t_\Delta$ of each $s_\Gamma$ since $\Delta=\Gamma|^a$.  By modifying the algorithm for \texttt{PDL} in \cite{Lang06}, we can get an algorithm to check \EPDL\ formula $\phi$ on $\calN^\bullet$ w.r.t.\ $\Vdash$, and its time complexity is $O(|\phi|^2\cdot |\calN^\bullet|^3)$. Thus, the time complexity of model checking $\phi$ on $\calM$ is bounded by $O(|\phi|^2\cdot |\S_\N|^3\cdot 2^{3|\S_\N|-3})$.
\end{proof}

We conjecture that the model checking problem of full \EPDL\ is \EXPTIME-complete, and leave the lower bound to the extended version of this paper.

\subsection{Complexity of conformant planning}

In the rest of this section, let us look at the complexity of conformant planning in terms of \EPDL\ model checking. Although the model checking problem of full \EPDL\ is likely to be \EXPTIME-complete, the complexity of model checking the \EPDL\ formula which encodes the conformant planning problem (cf. Theorem~\ref{theo:conformant_planning})
is in \PSPACE\ if the goal formula is program-free. More precisely, we can show the following:

\begin{theorem}
The problem of model checking \EPDL\ formulas in the shape of $\lr{(\Sigma_{a\in\B}(?K\lr{a}\top;a))^*}K\phi$, where $\phi$ is an epistemic formula (i.e. program-free) and $\B\subseteq \tA$, is in $\PSPACE$. \label{thm.mccp}
\end{theorem}

\begin{proof}
(Sketch)
Note that $(\sum_{a\in\tB}(?K\lr{a}\top;a))^*$ is a special program which has only simple epistemic tests depending  on the structure of the underlying Kripke model. Now given a Kripke model $\N$ and a set $\B\subseteq \tA$ we can define an ETS model $\N^\circ$ similar to $\N^\bullet$ but with a different definition for the action relations:
$$ 	\calR^\circ_a=\{(s_\Gamma,t_\Delta)\mid s\rel{a}t,\Delta=\Gamma|^a, \forall u\in\Gamma \exists v\ st.\ u\rel{a}v.\}$$

Note that the extra condition guarantees that the action $a$ is always executable w.r.t.\ the whole  $\Gamma$, thus fulfilling the test $?K\lr{a}\top$. Now we can have an analog of Corollary \ref{corollary:ETS_EPDL}, and reduce the problem of checking $\lr{\N,\U},s\vDash(\sum_{a\in\tB}(?K\lr{a}\top;a))^*K\phi$ to the reachability problem in $\N^\circ$: whether there is a path from $s_\U$ in $\N^\circ$ such that it can reach a state $t_{\U'}$ where $K\phi$ holds. Since $\phi$ is $[\pi]$-free, we can check it easily given $\U'$ using polynomial space, thus the main task is to find the reachable $t_{\U'}$. Note that, in the size of $\N$, there are exponentially many such $t_{\U'}$ and the maximal length of the plan is also exponential. However, we do not need to build the whole $\N^\circ$ and the bisection-like algorithm behind the proof of Savitch's Theorem will do the job.\footnote{A similar algorithm was used to pinpoint complexity of the conformant planning in AI, cf.\cite{KT2005}.} More precisely, we first pick up a $t_{\U'}$, and then run the recursive bisection method to see whether $t_{\U'}$ is reachable from $s_\U$ within $2^{|\N|}$ steps. The depth of the recursion is bounded by $log_2(2^{|\N|})=|\N|$ and at each recursion we need to record the choice of the state which can be encoded by a $(0,1)$-vector using $log_2(2^{|\N|})=|\N|$ space (plus one bit to record the result). Moreover, at the bottom of the recursion we only need to verify one step reachability, i.e., whether two states in $\N^\circ$ are linked by $\calR^\circ_a$,  without building the whole $\N^\circ$. Thus the whole procedure of model checking can be done using polynomial space.
\end{proof}

As we mentioned in the introduction, the conformant planning problems in the AI literature are usually given by using state variables and actions with preconditions and (conditional) effects, rather than explicit transition systems. The corresponding  explicit transition system can be generated by taking all the possible valuations of the state variables as the state space (an exponential blow up), and computing the transitions among the valuations according to the preconditions and the  postconditions of the actions. In terms of the size of explicit transition systems, our above result is  consistent with the $\EXPSPACE$ complexity result in the AI literature for conformant planning with Boolean and modal goals  \cite{KT2005,bonet2010conformant}. Actually,  the complexity result of Theorem~\ref{thm.mccp} can be strengthened to \PSPACE-complete based on the corresponding complexity result in the AI literature.

 However, not all the transition systems can be generated in this way since the preconditions and postconditions are (usually) purely propositional and thus two states that share the same valuation must have the same executable actions. In an arbitrary transition system,  multiple states with the same valuation may have different available actions due to some underlying protocol or other (external) factors not modelled by basic propositions.

 \label{sec.cp}

%
%
%

\section{Conclusions and future work}

\label{sec.conc}

In this work we first introduce the logical language \EAL\ over uncertainty maps and axiomatize it completely. \EAL\ is then extended to \EPDL\ with programs to specify conformant and conditional plans. We show that the conformant planning problems can be reduced to model checking problems of \EPDL. Finally we showed that model checking star-free \EPDL\ over uncertainty maps is \PSPACE-complete\ and model checking the full fragment is in \EXPTIME. On the other hand, model checking the conformant planning problem is in \PSPACE.

Note that our \EPDL\ is a powerful language which can already express conditional plans, for example, $(?p;a+?\neg p;b);c$. This suggests that we can use the very \EPDL\ language (\EPDL$^-$ is enough) to \textit{verify} plans in contingent planning w.r.t.\ a variant of the semantics which can handle feedbacks during the execution. In fact, observational power about the availability of the actions has been  already incorporated in \cite{wang2012}, which can be extended to general feedbacks discussed in the literature of contingent planning (cf.\ e.g., \cite{BonetG12a}). On the other hand, to check the \textit{existence} of a conditional plan, we are not sure whether \EPDL\ is expressive enough, as subtleties may arise as in the case of conformant planning. We leave the contingent planning to future work.

Another natural extension is to go probabilistic, and reduce the probabilistic planning over MDP to some model checking problem of the probabilistic version of our \EPDL. Our ultimate goal is to cast all the standard AI planning problems into one unified logical framework in order to facilitate careful comparison and categorization. We will then see clearly how the form of the goal formula, the constructor of the plan, and the observational ability matter in the theoretical and practical complexity of planning, in line with the research pioneered in \cite{BackstromJ11}.

\paragraph{Acknowledgement} Quan Yu is supported by NSF Grant No.61463044 and Grant No.[2014]7421 from the Joint Fund of the NSF of Guizhou province of China. Yanjun Li thanks the support from China Scholarship Council. Yanjing Wang acknowledges the support from ROCS of SRF by Education Ministry of China and the NSSF major project  12\&ZD119.

\bibliographystyle{eptcs}

\bibliography{MCEPDL}

\appendix
\section{Algorithms for EPDL$^-$}
\begin{definition}[Matrix representation]
	Let $\tB_{n\times m}$ denote a (0,1)-matrix of size $n\times m$. 
	A matrix $\tB_{n\times 1}$, or $\tB_n$ for short, is called a vector.
	Given finite uncertainty map $\calM$, its domain $\calS$ can be linearly ordered as $\{s_1,\cdots,s_n\}$. 
	Thus $\calM$ can be represented by a set $\{\tB^{\hspace{0.3mm}a}_{n\times n}\mid a\in\tA \}$ of adjacency matrices for accessibility relation, a vector $\tB^\calU_n$ for $\calU$ and a set $\{\tB^{\hspace{0.3mm}p}_n\mid p\in\tP\}$ of vectors for atomic propositions.
\end{definition}
\begin{definition}
	Given (0,1)-matrices $\tB'_{n\times k},\tB_{k\times m}$, 
	their product $\tB''_{n\times m}$ is defined as: $\tB''_{n\times m}[i,j]=1$ iff there exists $r\leq n$ such that $\tB'_{n\times k}[i,r]=\tB_{k\times m}[r,j]=1$ for all $1\leq i\leq n,1\leq j\leq m$.
\end{definition}
The following algorithms are to check whether $\phi$ holds on a pointed uncertainty map $\calM,s$ by Definition \ref{semanticsth}. The main algorithm (Algorithm \ref{MCAG})  recursively calls itself for each non-trivial subformula of $\phi$. The complex cases are for the subformulas in the form of $\lr{\pi}\phi$ and $K\phi$. By Definition~\ref{semanticsth}, to check $\calM,s\Vdash_\sigma\lr{\pi}\phi$, we need to make sure that there exists a sequence $\omega\in\L(\pi)$ and a state $t\in\calS$ such that $s\rel{\omega_\sigma}t$ and $\calM,t\Vdash_{\sigma r(\omega)}\phi$. Since $\pi$ is star-free, $|\omega|\leq |\pi|$ for each $\omega\in\L(\pi)$. It is clear that we cannot compute and store the whole set of $\L(\pi)$ within polynomial space. Instead, \textit{one by one} we generate all the possible sequences that are shorter than $|\pi|$ and are formed from the alphabet of $\pi$  (cf.\ line 14), and check whether they are in $\L(\pi)$. We can order the possible sequences lexicographically according to an ordering of the basic actions and tests in $Sig$, and compute the next sequence  merely from the current one using function \emph{next}. \emph{memb\_chec($\omega,\pi$)} checks whether it is the case $\omega\in\L(\pi)$. If $\omega\in \L(\pi)$, we need to check whether there exists $s_j\in\calS_\calM$ such that $s\rel{\omega_\sigma}s_j$ (Algorithm~\ref{CPW}) and $\calM,s_j\Vdash_{\sigma r(\omega)}\phi$, where $r(\omega)$ is the test-free subsequence of $\omega$ which is easy to compute.
For the case of $K\phi$, we need to calculate the state set $\calU|^\sigma$ (Algorithm~\ref{CCS}).

\section{Complexity analysis}
We suppose $|\S_\M|=n$ and $|\phi|=k$. Algorithm \ref{CCS} uses one variable $A$ to record the uncertainty set which requires $O(n)$ space.
 Note that there is a mutual recursion in Algorithm \ref{CPW} and \ref{MCAG}, but the depth of the overall recursion is bounded by $k$. In Algorithm~\ref{CPW}, the variable consuming the most of the space is the matrix $\tB_{n\times n}$ recording the (intermediate) relation. Since $\sigma$ and $\omega$ are also variables in the main algorithm and $|\omega|+|\sigma|\leq k$ due to the construction in Algorithm~\ref{MCAG}, the space usage of Algorithm~\ref{CPW} before the recursive calls of $PW$ and $MC$ is bounded by $O(k+n^2)$. 
For Algorithm~\ref{MCAG}, the most space-demanding part is the $\lr{\pi}\phi$ case, where we need to store $\pi$, $Sig$, and keep track one $\omega$ and one state $s$ in the loop, which are bounded by either $k$ or  $s$. Moreover, according to \cite{JR91}, the complexity of \emph{memb\_chec} is \textsc{NLOGspace}-complete in the size of $Sig$, i.e., the alphabet of $\pi$ which is bounded again by $k$. Thus before calling $MC$ and $PW$ again in the $\lr{\pi}\phi$ case, the space requirement is at most linear in both $k$ and $n$, which is less demanding than $PW$ for each recursion. Recall that the overall recursion depth of $MC$ (and $PW$) is bounded by $k$ thus the space usage of the whole algorithm is bounded by $O(k(k+n^2))=O(k^2+kn^2)$.

 \RestyleAlgo{ruled}\LinesNumbered
\begin{algorithm}\rm
	\SetAlgoVlined
	\SetKwData{Left}{left}\SetKwData{This}{this}\SetKwData{Up}{up}
	\SetKwFunction{Union}{Union}\SetKwFunction{FindCompress}{FindCompress}
	\SetKwInOut{Input}{input}\SetKwInOut{Output}{output}
	\Input{$\calU$, $\sigma$}
	\Output{$\tB^{\calU|^{\sigma}}_{n}$}
	\BlankLine
	$A \leftarrow \tB^{\U}_{n}$\;
	\For {$i\leftarrow 1$ \KwTo $|\sigma|$}
	{
		$A \leftarrow A \times \tB^{\sigma[i]}_{n\times n}$\;
	}
	
	\Return $A$\;
	\caption{Function CNU$(\calU,\sigma)$: Calculate the the new uncertainty set $\calU|^\sigma$}\label{CCS}
\end{algorithm}

\RestyleAlgo{ruled}\LinesNumbered
\begin{algorithm}\rm
	\SetAlgoVlined
	\SetKwData{Left}{left}\SetKwData{This}{this}\SetKwData{Up}{up}
	\SetKwFunction{Union}{Union}\SetKwFunction{FindCompress}{FindCompress}
	\SetKwInOut{Input}{input}\SetKwInOut{Output}{output}
	\Input{computation sequence $\omega$, action sequence ${\sigma}$}
	\Output{$\tB_{n\times n}$}
	\BlankLine
	\Switch{$\omega_{\sigma}$}{
		\lCase{$\epsilon_{\sigma}$}
		{			
			\Return $Matrix(\{(s,s)\mid s\in\calS\})$ \tcc*[f]{$Matrix(R)$ is the $(0,1)$-matrix representation of the binary relation $R$}
		}
		\lCase{$(?\phi\omega')_{\sigma}$}
		{
			\Return $Matrix(\{(s,s)\mid$ MC$(\calM, s, \sigma, \phi) = $ true $\})\times$ PW($\omega',\sigma$)
		}
		\lCase{$(a\omega')_{\sigma}$}
		{
			\Return $\tB^a_{n\times n}\times$ PW($\omega',{\sigma a}$)
		}
	}
	\caption{Function $PW(\omega,{\sigma})$: Calculate the binary relation $\rel{\omega_{\sigma}}$}\label{CPW}
	
\end{algorithm}
\RestyleAlgo{ruled}\LinesNumbered
\begin{algorithm}\rm
	\SetKw{Not}{not}
	\SetKw{And}{and}
	\SetAlgoVlined
	\SetKwData{Left}{left}\SetKwData{This}{this}\SetKwData{Up}{up}
	\SetKwFunction{Union}{Union}\SetKwFunction{FindCompress}{FindCompress}
	\SetKwInOut{Input}{input}\SetKwInOut{Output}{output}
	\Input{The pointed uncertainty map $(\calM, s)$, sequence of actions $\sigma$, $\phi\in \tEPDL^{-}$.}
	\Output{true if $\calM, s \Vdash_\sigma \phi$.}
	\BlankLine
	\Switch{$\phi$}{

		\uCase{$\langle \pi \rangle \varphi$}
		{
			Let $Sig$ be the array consisting of atomic programs and formulas in $\pi$ ordered according to their first appearances\;
			$\omega\leftarrow Sig[1]$ \tcc*[f]{$\omega$ is the candidate sequence we want to test}\;
			\While{$|\omega|\leq|\pi|$}{
						\If{memb\_chec($\omega,\pi$)}{
							\For{i = 1 \KwTo $\calS_{\M}$}
							{
								\If{
									$(s,s_i)\in PW(\omega,\sigma)$
									 }
								{				
									\lIf{
										MC$(\calM,s_{j},\sigma r(\omega),\varphi)$
										}{
										\Return true
										}
									
								}
							}
							}	
					$\omega\leftarrow$ \textit{next}($\omega, Sig$) \tcc*[f]{calculate the next sequence lexicographically according to the order $Sig$}\;

			}
			\Return false\;
		}
		\uCase{$K\varphi$}
		{
			$\tB^{\calU|^{\sigma}}_{n}$ = CNU$(\calU,\sigma)$ \tcc*[f]{calculate the vector representation of $\calU|^{\sigma}$}
			
			\For{m = 1 \KwTo $|\calS_{\calM}|$}
			{
				\lIf{$(\tB^{\hspace*{0.5mm}\calU|^{\hspace*{0.5mm}\sigma}}_{n})_m = 1$ {\bf and} MC$(\calM, s_{m},\sigma,\varphi) = $ false}
				{
					\Return false
				}
			}
			
			\Return true\;
		}
	}
	
	\caption{Function MC$(\calM,s,\sigma,\phi)$: Model checking algorithm for $\tEPDL^{-}$ (Boolean cases omitted)}\label{MCAG}
	
\end{algorithm}

\end{document}